\documentclass{article}

\usepackage[nonatbib,preprint]{neurips_2020}

\usepackage{tcolorbox}

\usepackage{amsmath,amsfonts,amssymb,amscd,amsthm,xspace}

\usepackage{epsfig}

\usepackage{enumitem}
\usepackage{bbm}
\usepackage{dblfloatfix}
\usepackage{bm}
\usepackage{color}
\usepackage[utf8]{inputenc} 
\usepackage[ruled]{algorithm2e}
\usepackage{wrapfig}
\usepackage[T1]{fontenc}    
\usepackage{hyperref}       
\usepackage{url}            
\usepackage{booktabs}       
\usepackage{amsfonts}       
\usepackage{nicefrac}       
\usepackage{microtype}   
\usepackage{graphicx}
\graphicspath{{Figs/}}
\newcommand{\nn}{n}
\newcommand{\nm}{m}
\usepackage{todonotes}
\usepackage{wrapfig}
\usepackage{lipsum}

\usepackage{cite}

\newtheorem*{theorem*}{Theorem}
\newtheorem{theorem}{Theorem}
\numberwithin{theorem}{section}

\numberwithin{corollary}{section}

\theoremstyle{definition}
\newtheorem{definition}{Definition}
\numberwithin{definition}{section}

\numberwithin{example}{section}

\newtheorem{remark}{Remark}

\newtheorem{assumption}{Assumption}

\title{Disentangling Observed Causal Effects from Latent Confounders using Method of Moments}

\author{%
Anqi Liu\textsuperscript{1}* 
  \quad
  Hao Liu\textsuperscript{1}*
   \quad
   Tongxin Li\textsuperscript{1}*
   \quad  
   Saeed Karimi-Bidhendi\textsuperscript{2} \\
   \quad
   \bf{
  Yisong Yue\textsuperscript{1}  
   \quad
  Anima Anandkumar\textsuperscript{1} 
  }\\
  \textsuperscript{1}California Institute of Technology \
  \textsuperscript{2} UC Irvine \\
  \texttt{\{anqiliu,hliu3,tongxin,anima,yyue\}@caltech.edu  \ skarimib@uci.edu }
}

\begin{document}

\maketitle

\begin{abstract}
Discovering the complete set of causal relations among a group of variables is a challenging unsupervised learning problem. Often, this challenge is compounded by the fact that there are latent or hidden confounders. When only observational data is available, the problem is ill-posed, i.e. the causal relationships are non-identifiable unless strong modeling assumptions are made. When (soft) interventions are available, we provide guarantees on identifiability and learnability under mild assumptions. We assume a linear structural equation model (SEM) with independent latent factors and directed acyclic graph (DAG) relationships among the observables. Since the latent variable inference is based on independent component analysis (ICA), we call this model SEM-ICA. We use the method of moments principle to establish model identifiability. We develop efficient algorithms based on coupled tensor decomposition with linear constraints to obtain scalable and guaranteed solutions.  Thus, we provide a principled approach to tackling the joint problem of causal discovery and latent variable inference.

\end{abstract}

\section{Introduction}

\label{sec:prev}
Discovery of causal relationships among a group of  variables lies at the heart of scientific discovery \cite{Pearl2009}. Inferring causes and effects, and separating them from spurious correlations, is central to scientific domains such as medical drug trials, high energy physics, and econometrics.  From a machine learning perspective, discovery of causal effects can lead to more robust, reliable, and interpretable predictions. 

Causal inference aims to learn strength of causal interactions on a known graph, while causal discovery also aims to learn the causal graph (and the parameters) from data. This problem is challenging especially in the presence of hidden or latent confounders, which make the problem ill posed. These are variables that are not directly observed but interact with the observables. These hidden variables cause spurious associations among the observed  variables and make it challenging to differentiate causal relations from spurious correlations. For instance, in drug trials, there are many unobserved factors that affect drug efficacy in a patient (e.g. their lifestyle choices), and not all of that is recorded.
As another example, to do system identification in physical systems, like control systems, not all the contributors to the dynamics are observed. This can be due to the difficulty in data collection or the limitations in data fieldity \cite{shi2019neural}. On the other hand, it is feasible to intervene or manipulate the physical system to record different kinds of data. Using such interventions to learn causal relationships is an interesting problem.





The problem we consider is two-fold. We aim to simultaneously discover the relationships among a group of observed variables, and in addition, discover latent factors and their relationships with the observables. We aim  to solve both causal structure discovery and latent variable inference.


\newpage
\paragraph{Causal Discovery with Interventions} While observational data is easy to obtain, causal relationships cannot always be uniquely identified using only observational data. 
Interventions provide essential information for discovering the causal effects. For example, in structural equation models (SEMs)~\cite{Silva:2006:LSL:1248547.1248555}, the act of assigning a variable $X$ to a specific value $x$ allows us to dislodge $X$ from its prior causes and evaluate the expected change of other variables. More generally, softly intervening on a subset of variables induces a new probability distribution over the variables~\cite{Pearl2009}. This has been utilized previously  for estimating causal effects when the confounders are also observed variables~\cite{Pearl2009}. What we consider is the general interventional framework in this paper. Interventions are usually represented using $\mathbf{Do}$ operator. For example, in Figure \ref{fig:do}, $\mathbf{Do(X_1)}$ represents intervention on variable $X_1$. As a consequence, the incoming 
edge from $H$ to $X_1$ will be removed after the intervention.
\begin{wrapfigure}[9]{r}{0.45\textwidth}
      \includegraphics[height=2.45cm, trim={4cm, 8cm, 9cm, 8cm},clip]{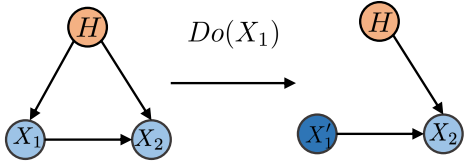}
    \caption{The Do operator. }
    \label{fig:do}
\end{wrapfigure}


\paragraph{Causal Discovery with hidden confounders}  Previous works that  focus on learning causal models with hidden confounding variables generally consider only  observational data ~\cite{hoyer2009nonlinear, zhang2012identifiability,1908.03932}.  However, these works  do not consider information obtained through interventions and hence, require strong modeling assumptions. Hence, they may not be applicable in many settings.  For instance, the linear, non-Gaussian, and acyclic model (LiNGAM) has been investigated for causal discovery using the independent component analysis (ICA) model under non-gaussian noise assumptions on observational data \cite{shimizu2006linear,shimizu2011directlingam}. It is recently extended to the hidden confounder 
setting\cite{1908.03932}. However, using observational data, it can only recover a class of causal models that are equivalent and compatible with observational data, but this does not guarantee unique identifiability of the causal relationships.

\paragraph{Latent variable models} In terms of latent variable modeling without causal relationships, there is a rich body of work proposing different classes of models and algorithms. We consider the independent component analysis (ICA) which assumes the presence of independent latent factors that mix together to generate the observed data~\cite{Hyvarinen:2000:ICA:351654.351659}. Learning the mixing process provides a natural  disentangled representation of the observed domain ~\cite{khemakhem2019variational}.
The simplest setting is when the mixing process is linear. In this case, the mixing matrix can be recovered provably and efficiently under mild assumptions using tensor decomposition methods and higher order (3rd or 4th) moments \cite{JMLR:v15:anandkumar14b}.

In contrast to tensor methods, standard approaches for learning latent variable models such as expectation maximization (EM) are not guaranteed to converge to the globally-optimal solution and are prone to failure in high dimensions. We consider the method of moments approach that  fits observed moments in order to estimate the model. For a wide class of   latent variable models that includes ICA, \cite{JMLR:v15:anandkumar14b} 
showed that low order moments (up to third or fourth) is sufficient to provably learn the model. Efficient tensor decomposition algorithms have been developed that guarantee learning with polynomial sample and computational complexity. 

\paragraph{Our formulation} In this paper, we propose a new model termed as SEM-ICA that  combines the two frameworks, viz., linear SEMs and ICA, to incorporate both observables and latent confounders. We use a mixing matrix to represent the confounding effects between the latent variables and the observable variables. We use a DAG graph to represent the relationship among observable variables.  Thus, we assume a latent generative process, given by the  ICA model, and this acts as latent confounders to the causal discovery process among the observables. This is a reasonable assumption in many scenarios. For example, in power networks, our intervention on different distributors would not change the latent variables such as unrecorded weather or temperature.

We assume that (soft) interventions are available and this helps disentangle the causal and latent effects. A key observation we make is the fact that there is invariance of the latent effects across different interventions, since we assume that they are driving the generative process of observations. This insight allows us to relate the data obtained  under different interventions. Our method generally allows for both hard interventions and soft interventions. But we focus our analysis and simulation on the hard interventions.

\paragraph{Algorithmic Contributions} We establish a novel algorithm for learning linear SEM-ICA based on the method of moments. We consider moments computed both from observational and interventional data. In particular, each of these datasets can be expressed as another ICA model. Thus, we cannot hope to learn the correct underlying model by considering only observational data. In particular, we need to understand how the different datasets from different interventions are related to one another. We show that the mixing matrices of these models are related to the parameters of the underlying SEM-ICA model through a linear system of equations. We assume that the distribution of the generative latent factors is invariant under interventions. This holds in many cases when environmental conditions being invariant under controlled interventions. For example, in power networks, our intervention would not change the latent variables such as unrecorded weather or temperature.  We then formulate a joint optimization method that incorporates both the tensor decomposition and the linear systems with invariance constraints to obtain recovery of the mixing and observed matrices. 
\vspace{-5pt}

\paragraph{Theoretical Guarantees} We analyze a variant of our algorithm and prove  identifiability for the mixing and observed causal matrices. We demonstrate that under mild non-degeneracy conditions our method can identify both the mixing and observed causal matrices up to permutation, sign and scale. This requires the mixing matrix of the ICA model to be full-rank and the adjacent graph among observables to be directed acyclic graph(DAG). Both are natural for identifiability, which implies that the latent factors are not (linearly) degenerate.
We also establish the learning error guarantees for the whole algorithm and show that we achieve polynomial sample and computational complexity. This is the first work that provides guaranteed causal discovery with hidden confounders.

\vspace{-5pt}
\vspace{-5pt}


\section{Problem Setup}
\label{sec:research}

\begin{figure}[t]
        \centering
        \setlength{\tabcolsep}{-5pt}
        \begin{tabular}{ccc}
          \includegraphics[height=4cm, trim={7cm, 6cm, 6.5cm, 6cm},clip]{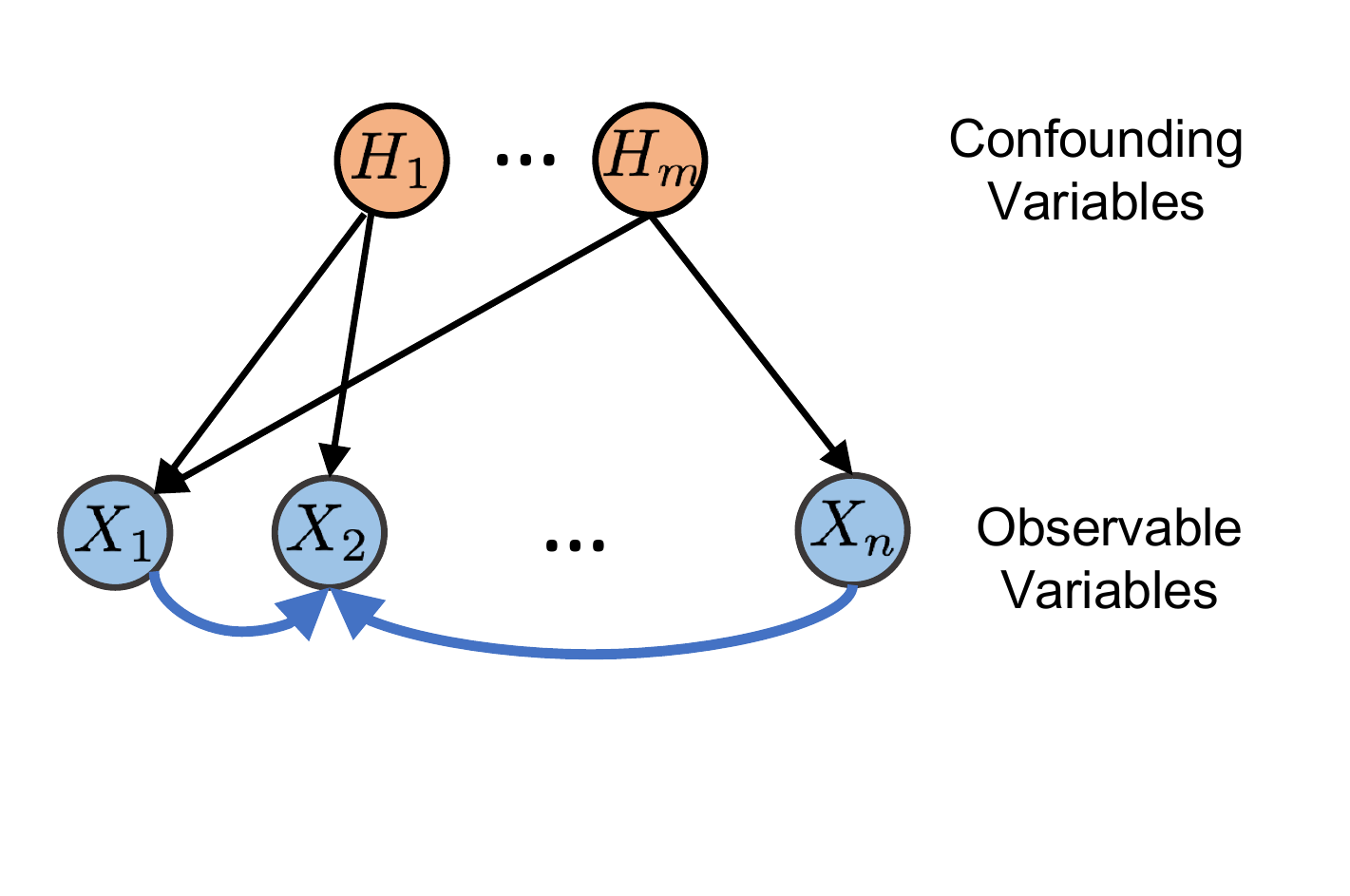}&
           \includegraphics[height=4cm, trim={7cm, 6cm, 10cm, 6cm},clip]{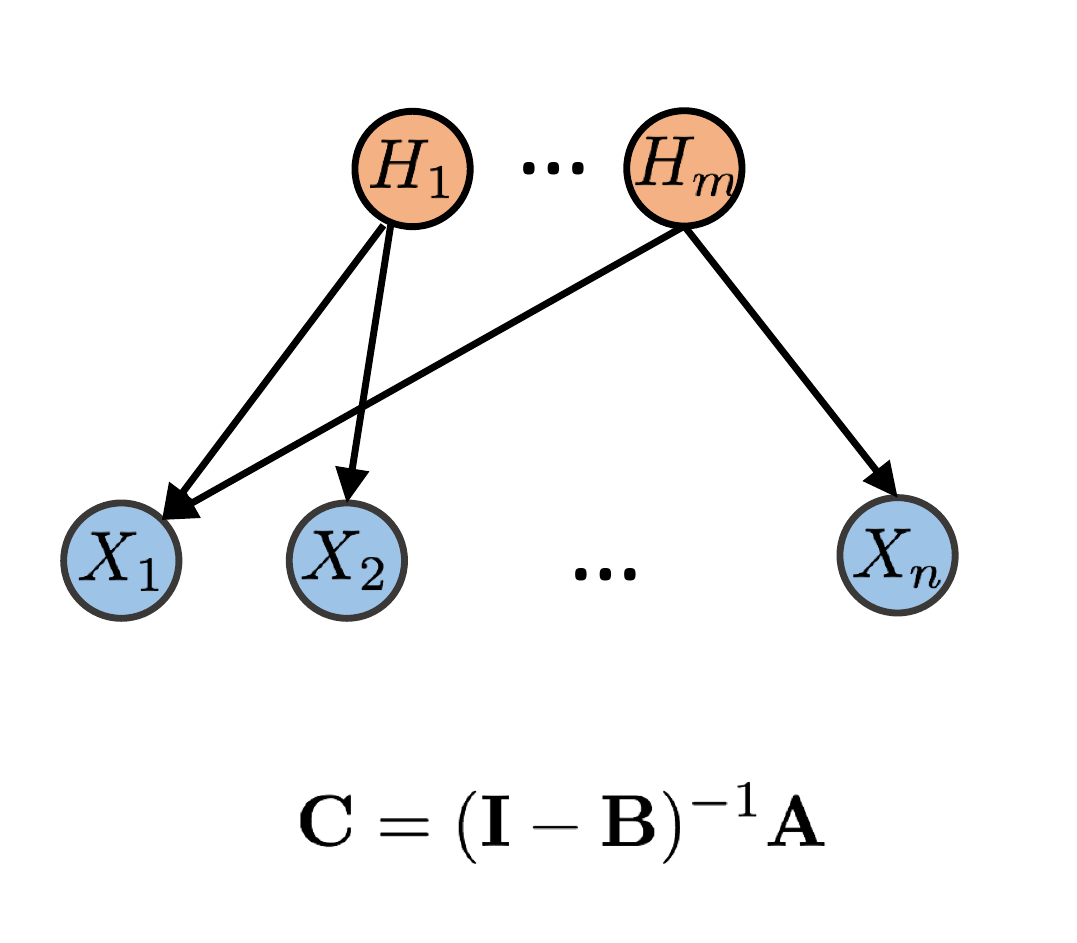}&
           \includegraphics[height=4cm, trim={7cm, 6cm, 10cm, 6cm},clip]{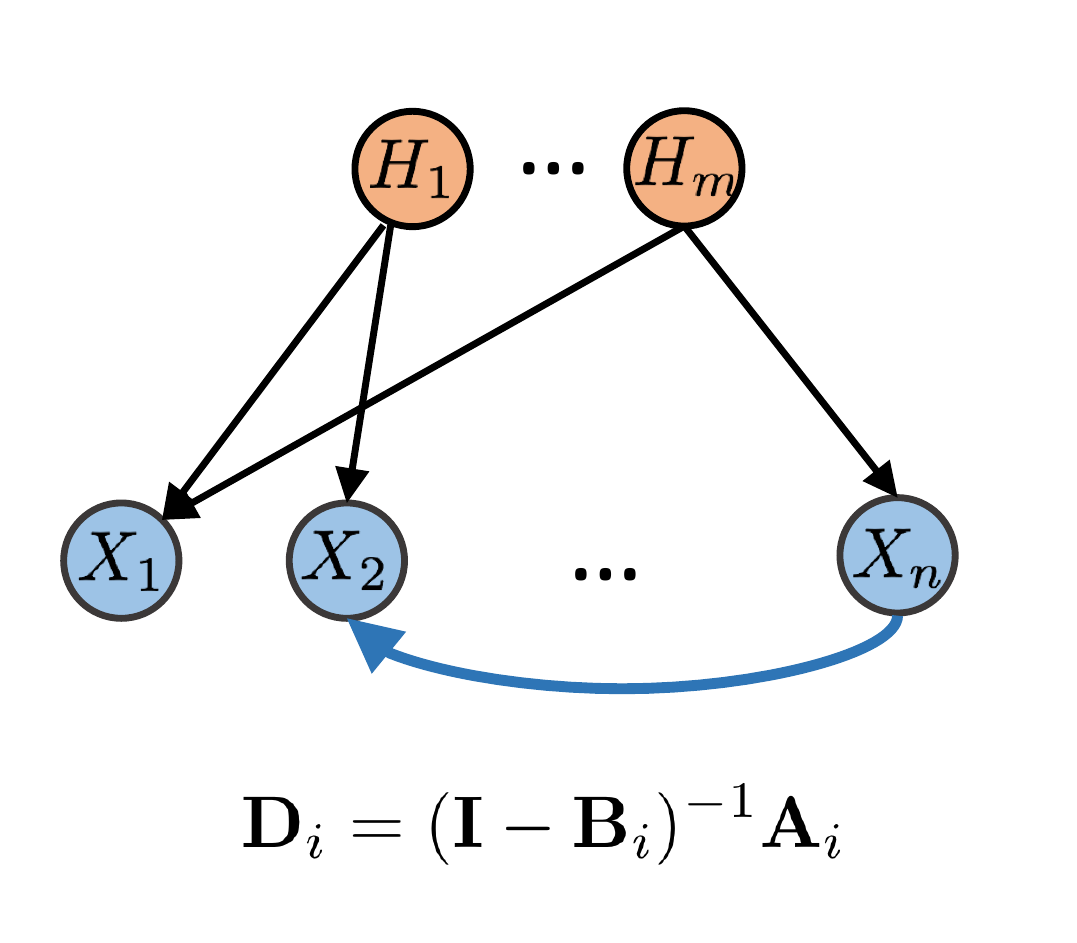} \\
           (a) & (b) & (c)
       \end{tabular}
       
         \caption{(a)  Graphical representation of the SEM-ICA model. The orange-circled nodes are the confounding variables that are mutually independent. The nodes in blue stand for observable variables and they are allowed to have causal relations between each other;
         (b) Graphical representation of the ICA model $\mathbf{X} = \mathbf{C}\mathbf{H}$. The coordinates of $\mathbf{H}$ are independent; (c) Graphical representation of the SEM-ICA model after intervene $i$th variable.}
        \label{fig:ICA}
     
    \end{figure}

\vspace{-5pt}
\vspace{-5pt}
Given a set of $\nn$ \textit{observable}  variables $\mathcal{X}:=\{X_1,\ldots,X_{\nn}\}$ and a set of $\nm$ \textit{confounding}  variables $\mathcal{H}:=\{H_1,\ldots,H_{\nm}\}$, we consider a directed graph $G=(\mathcal{X}\bigcup\mathcal{H},E)$ where $\mathcal{X}\cup\mathcal{H}$ is the vertex set. The edge set $E$ consists of all directed edges between vertices, which represent the causal relations between two neighboring variables. For a vertex $X\in\mathcal{X}$, the parent set of $X$ is denoted by $\mathrm{Pa}(X) \subseteq\mathcal{X}\bigcup\mathcal{H}$ such that $X' \in \mathrm{Pa}(X)$ if and only if there is a direct edge from $X'\in\mathcal{X}\bigcup\mathcal{H}$ to $X$. 
For modeling the causality among confounding variables and observable variables, we assume that the confounding variables are independent and there only exist paths from confounding variables to observable variables, but not the other way around. Moreover, we assume there is no self-loop and the partial graph with vertices in $\mathcal{X}$ is acyclic, whence can be regarded as a DAG. Therefore, the graph $G$ is a quasi-bipartite graph, as shown in Figure~\ref{fig:ICA} (a).

 We fix the notation for this paper. We use boldface letters $\mathbf{X}$ and $\mathbf{H}$ for random vectors and denote the $i$-th entry of $\mathbf{X}$ as $\mathbf{X}_i$. Observed measurements are realizations of $\mathbf{X}$, written as $\mathbf{X_{[1]}},\ldots,\mathbf{X_{[N]}}$ and $N$ is the number of samples. Matrices are denoted by uppercase boldface letters $\mathbf{A}$, $\mathbf{B}$, $\mathbf{C}$ and $\mathbf{D}$. For a matrix $\mathbf{A}$, the $i$-th row is written as $\mathbf{a}_i$ and the $i$-th column is denoted by $\mathbf{a}^{(i)}$. The $(i,j)$-th entry of $\mathbf{A}$ is denoted by $a_{i,j}$. The superscript $T$ refers to the transpose of a matrix or vector.

\subsection{Structural Equation Models with Latent Variables}

The set of observable variables is represented by $\mathbf{X}$ as a vector in $\mathbbm{R}^{\nn}$. The \textit{confounding (hidden) variables} are denoted by a vector $\mathbf{H}$ in $\mathbbm{R}^{\nm}$. We consider a linear model
for the observable and confounding variables:
\begin{align}
\label{eq:1}
\mathbf{X} = \mathbf{A}\mathbf{H} + \mathbf{B} \mathbf{X} + \mathbf{N}
\end{align}
wherein $\mathbf{A}\in \mathbbm{R}^{\nn\times\nm}$ denotes the \textit{mixing matrix}. It represents the correlation between observable and confounding variables (and equivalently, the edges between confounding variables and observable variables in Figure \ref{fig:ICA}). Similarly, the matrix $\mathbf{B}\in\mathbbm{R}^{\nn\times\nn}$ is a \textit{observed causal matrix}, which encodes the causal relationship among observable variables (corresponding to directed edges in Figure \ref{fig:ICA}). Additive noise is written as $\mathbf{N}$.
Rearranging the terms, we obtain the following alternative expressions:
\begin{align}
\label{eq:2}
  \mathbf{X} = \mathbf{C}\mathbf{H} +\overline{\mathbf{N}}.
\end{align}
where $\mathbf{C}:=(\mathbf{I}-\mathbf{B})^{-1}\mathbf{A}$ and $\overline{\mathbf{N}}:=(\mathbf{I}-\mathbf{B})^{-1}\mathbf{N}$.
Here, $\mathbf{I}$ is an $\nn\times\nn$ identity matrix. Below we present the assumption of our model.
\begin{assumption}
\label{ass:1}
We assume that $\mathbf{A}$ is a full-rank matrix. We also suppose that the faithfulness assumption holds so that $a_{j,k}\neq 0$ if there is a causal path from $\mathbf{H}_k$ to $\mathbf{X}_j$ and $b_{i,j}\neq 0$ if there is a causal path from $\mathbf{X}_j$ to $\mathbf{X}_i$, for all $i,j=1,\ldots,\nn$ and $k=1,\ldots,\nm$.
\end{assumption}
With the help of $\nn$ interventions, we can discover the causal ordering of $X_1,X_2,...,X_n$ (see~(\ref{causal_order})). Hence after suitable permutation of the rows and columns of $\mathbf{B}$, we can formulate it as strict lower triangular and the matrix $\mathbf{C}$ is full-rank. In this work, we consider a new approach for identifying causal effects, \textit{i.e.,} recovering the mixing matrix $\mathbf{A}$ and observed causal matrix $\mathbf{B}$, by conducting interventions, defined in the next section.

\subsection{Interventions}

An intervention in causal discovery refers to the experiment that sets the values of the variables and isolates them from their causes. 

\begin{definition}
 {(\bf $\mathbf{Do}(\cdot)$ operator)}: We write $\mathbf{Do}(\cdot)$ as the intervention operator. Graphically, each ideal intervention $\mathbf{Do}(\mathbf{X}_i = \lambda_i)$  corresponds to removing all the directed edges from $\mathrm{Pa}(\mathbf{X}_i)$ to $X_i$ and set the coordinate $\mathbf{X}_i$ in $\mathbf{X}$ to be $\lambda_i$. We write $\mathbf{X}_i\in\mathbbm{R}^\nn$ as a vector with the $i$-th entry set to $\lambda_i$.

\end{definition} 
After an intervention $\mathbf{Do}(\mathbf{X}_i = \lambda_i)$, we assume $\mathbf{Y}_{-i} = \mathbf{X}_{-i} | \mathbf{Do}(\mathbf{X}_i = \lambda_i)$ and $Y_i = \lambda_i$,  the SEM defined in (\ref{eq:1}) still holds by modifying the mixing matrix and observed causal matrix as follow:
\begin{align}
\label{eq:intervened_system}
\mathbf{Y}_i = \mathbf{A}_i\mathbf{H} + \mathbf{B}_i\mathbf{Y}_i+ \mathbf{N}
\end{align}
where $\mathbf{A}_i$ and $\mathbf{B}_i$ are formed by setting the $i$-th rows of $\mathbf{A}$ and $\mathbf{B}$ as the all-zero vector and the unit vector $\mathbf{e}_i$. Rearranging the terms, similarly we obtain $\nn$ linear systems by intervening $\mathbf{Do}(\mathbf{X}_i = \lambda_i)$ for $i=1,\ldots,\nn$ and obtain the following alternative expressions:
\begin{align}
\label{eq:3}
  \mathbf{Y}_i = \mathbf{D}_i\mathbf{H}+ \overline{\mathbf{N}}_i, \ \text{ for } \ i=1,\ldots, \nn
\end{align}
where each matrix $\mathbf{D}_i:=(\mathbf{I}-\mathbf{B}_i)^{-1}\mathbf{A}_i\in\mathbbm{R}^{\nn\times\nm}$ is defined as the \textit{response matrix} to the intervention $\mathbf{Do}(\mathbf{X}_i = \lambda_i)$ and $\overline{\mathbf{N}}_i:=(\mathbf{I}-\mathbf{B}_i)^{-1}\mathbf{N}$. 

The goal of this work is to use both the observational data and the experimental data obtained by interventions $\mathbf{Do}(\mathbf{X}_i = \lambda_i)$ to identify the mixing matrix $\mathbf{A}$ and the observed causal matrix $\mathbf{B}$ in (\ref{eq:1}), under the assumption presented in Assumption~\ref{ass:1}.
In the following section, we discuss independent component analysis ( ICA) models and the associated tensor decomposition methods, that can be utilized for recovering the response matrices $\mathbf{C}, \mathbf{D}_1,\ldots,\mathbf{D}_{\nn}$, based on observed samples of $\mathbf{X}, \mathbf{Y}_{1},...,\mathbf{Y}_{\nn}$.

\subsection{Recovery of Causal Order} 
\label{causal_order}
It is well known that if single-variable interventions are conducted on the variables, causal ordering can be recovered \cite{hyttinen2012learning} \cite{Eberhardt2006}. Intuitively, by doing single-variable interventions, the causal ordering can be recovered by checking the variables that have changed after the interventions. Previous work provides method that can be used to recover the causal directions with $n-1$ interventions \cite{Eberhardt2006}. In our framework, we can recover the causal ordering immediately after getting the experimental data by doing interventions $\mathbf{Do}(\mathbf{X}_i=\lambda_i)$ on the variables $\mathbf{X}_1,\mathbf{X}_2,...,\mathbf{X}_n$. Furthermore, we can also determine matrices $\mathbf{A}$ and $\mathbf{B}$ using the same $n$ single-variable interventions as shown below. 

\section{Tensor Decomposition for ICA}
\label{sec:ica}

\subsection{Independent Component Analysis}

The classical independent component analysis (ICA) model is a well-known blind source separation method for separating a set of source signals from a set of mixed signals, with no or a little prior knowledge about the source signals or mixing process. ICA-based methods decompose a multivariate random signal into a set of independent non-Gaussian signals. Therefore, they can be applied to the linear structural equations in (\ref{eq:2}) and (\ref{eq:3}). Taking the equation in (\ref{eq:2}) as an example, our goal is to recover $\mathbf{C}$ whose columns are defined as \textit{independent components} (ICs), provided with data samples of $\mathbf{X}$. The corresponding graphical representation is depicted in Figure~\ref{fig:ICA}. 
The classical ICA problem is an ill-posed task since it has inevitable indeterminacy of permutation and dilation of the independent components. Therefore, without additional information, we can only recover columns of the matrix $\mathbf{C}$ in equation (\ref{eq:2}) up to permutation and scaling. Under Assumption~\ref{ass:1}, the matrix $\mathbf{C}$ in equation (\ref{eq:2}) is full-rank. Thus, it is identifiable if latent variables is non-Gaussian.

\subsection{Tensor Decomposition}
\label{sec:tensor}
Various techniques have been proposed to recover the factor matrix $\mathbf{C}$ of an ICA model $\mathbf{X}=\mathbf{C}\mathbf{H}+\overline{\mathbf{N}}$. In this work, we use the method in~\cite{JMLR:v15:anandkumar14b} for guaranteed recovery of ICs. This procedure, which we refer to as the \textit{tensor method}, uses higher-order moments of the random vector $\mathbf{X}$ to recover the columns of $\mathbf{D}$. Here, we use the $4$-th order cumulant of zero-mean vector $\mathbf{X}$ defined as:
\begin{align}
\label{eq:tensor}
    \mathbf{M}_4 = \mathbbm{E}\left[\mathbf{X}\otimes\mathbf{X}\otimes\mathbf{X} \otimes\mathbf{X}\right] - T 
\end{align}
where $T$ is the following $4$-th order tensor, $\text{ for } 1\leq i_1,i_2,i_3,i_4 \leq \nm$:
\begin{align*}
[T]_{i_1,i_2,i_3,i_4} =&\mathbbm{E}\left[\mathbf{X}_{i_1}\mathbf{X}_{i_2}\right]\mathbbm{E}\left[\mathbf{X}_{i_3}\mathbf{X}_{i_4}\right]
+\mathbbm{E}\left[\mathbf{X}_{i_1}\mathbf{X}_{i_3}\right]\mathbbm{E}\left[\mathbf{X}_{i_2}\mathbf{X}_{i_4}\right]
+\mathbbm{E}\left[\mathbf{X}_{i_1}\mathbf{X}_{i_4}\right]\mathbbm{E}\left[\mathbf{X}_{i_2}\mathbf{X}_{i_3}\right].
\end{align*}
Denote by $\kappa_i$ the \textit{excess kurtosis}, a measure of non-Gaussianity
defined as
\begin{align}
\nonumber
\kappa_i := \mathbbm{E}\left[\mathbf{H}_i^4\right] - 3,
\end{align}
it follows that
$\mathbf{M}_4 = \sum_{i=1}^{\nm}\kappa_i \mathbf{c}^{(i)}\otimes  \mathbf{c}^{(i)}\otimes  \mathbf{c}^{(i)}\otimes  \mathbf{c}^{(i)}$.
The $4$-th order moment $\mathbf{M}_4$ can be estimated from observed measurements and tensor decomposition algorithms can be applied to recover the columns of $\mathbf{C}$. In particular, the alternating power updates proposed in~\cite{JMLR:v15:anandkumar14b} can be applied to solve the tensor decomposition problem in~(\ref{eq:tensor}).
\vspace{-5pt}
\section{Causal Effects Identification}
\label{sec:causal_effect}


Based on the causal ordering restored, the matrix $\mathbf{B}$ after suitable permutation of the rows and columns becomes a lower triangular matrix. Therefore, without loss of generality, we assume that $\mathbf{B}$ is strict lower triangular and thus matrix $\mathbf{C}$ is full-rank. 
The next step is to recover the mixing matrix $\mathbf{A}$ and observed causal matrix $\mathbf{B}$, from the linear systems:
 $(\mathbf{I}-\mathbf{B}_i)\mathbf{D}_i = \mathbf{A}_i, \quad \text{for } i=0,\ldots,\nn$, which is
derived from (\ref{eq:3}) and (\ref{eq:2}). 

\subsection{Recovery of Response Matrices}
\label{PS_ind}
To obtain matrices $\mathbf{D}_0,\ldots,\mathbf{D}_\nn$ satisfying above, the labels of rows and columns of the matrices should be preserved, \textit{i.e.,} the rows and columns should correspond to the same confounding or observable variables. However, the tensor decomposition method reported in~\cite{JMLR:v15:anandkumar14b} leaves the permutation and scaling ({\textit{PS}}) for columns of any response matrix unknown. We name this issue {\textit {PS indeterminacy}}. Without considreation of time complexity, we can solve this issue in a brute force manner, by trying out different permutations of the hidden vectors since only the original mixing and observed causal matrices make the entire linear system solvable. Hence this issue will not affect the identification of our method. Also, heuristic ways exist for preserving the correspondence between columns of two response matrices $\mathbf{D}_i$ and $\mathbf{D}_j$ in practice. With prior knowledge, adding regularizers can also eliminate indeterminacy in an ICA model (\cite{lu2003eliminating}).


\subsection{Recovery of Mixing and Observed Causal Matrices}

For robust identification of $\mathbf{A}$ and $\mathbf{B}$, instead of solving the linear systems in section (~\ref{sec:causal_effect}), we combine the optimization objective in the tensor methods solving ICA with the linear systems to formulate a joint optimization objective: 
\begin{subequations}
\begin{align}
\label{eq:opt_1_1}
\min_{\mathbf{A},\mathbf{B}, \mathbf{C}, \mathbf{D}_i} \ \  & \left\|(\mathbf{I}-\mathbf{B})\mathbf{C}-\mathbf{A}\right\|_F^2  +  \sum_{i=1}^{\nn} \left\|(\mathbf{I}-\mathbf{B}_i)\mathbf{D}_i-\mathbf{A}_i\right\|_F^2 \\ &+ \| K_4(\mathbf{C}) - \mathbf{M}_4(\mathbf{X}) \|_F^2   + \sum_{i=1}^\nn \| K_4(\mathbf{D}_i) - \mathbf{M}_4(\mathbf{Y}_i) \|_F^2\\
\text{s.t. } \ \ & \mathrm{rank}\left(\mathbf{A}\right) = \min\{\nm,\nn\},
\quad \mathbf{A}\in\mathbbm{R}^{\nn\times\nm},\label{eq:opt_2_1}
\quad\mathbf{B}\in\mathbbm{R}^{\nn\times\nn}.
\end{align}
\end{subequations}
Here, $K_4$ is the Kronecker product. In practice, we use Tensorly library \cite{kossaifi2019tensorly} to jointly optimize $\mathbf{D}_i$ together with $\mathbf{A}$ and $\mathbf{B}$. Specifically, the library utilizes the alternative least square (ALS) method for solving the optimization problems. In practice, we can also add regularization to match columns of $\mathbf{D}_i$ and $\mathbf{C}$ to eliminate the optimization difficulties brought by permutation. We outline the recovery process in Algorithm~\ref{alg:1}.

\begin{algorithm}[h]
\medskip
\KwData{Observed measurements (as copies of $\mathbf{X}$) and model parameters} 
\KwResult{Estimated mixing and observed causal matrices $\widehat{\mathbf{A}}$ and $\widehat{\mathbf{B}}$}
\medskip
\hrule
\hrule
\medskip
\ {\it Step 1.} Recover the matrix $\mathbf{C}$ (using tensor decomposition methods).

\ {\it Step 2.} Obtain response matrices via interventions:
\For {$i = 1 \ to \ n $}{
 Perform the interventions $\mathbf{Do}$($\mathbf{X}_i$ = $\lambda_i$).
 
 Collect experimental measurements (as copies of $\mathbf{Y}_i$).
 
 Recover the response matrix $\mathbf{D}_i$ (using tensor decomposition methods).
}

\ {\it Step 3.} Recover causal ordering.


\ {\it Step 4.} Solve~\eqref{eq:opt_1_1}-\eqref{eq:opt_2_1}.

\Return{\text{Estimated $\widehat{\mathbf{A}}$ and $\widehat{\mathbf{B}}$}}
	\medskip
\caption{Algorithm for recovering mixing and observed causal matrices.}
\label{alg:1}
\end{algorithm}

\vspace{-5pt}
\vspace{-5pt}
\vspace{-5pt}

\section{Recovery Guarantees}
\label{sec:analysis}

With the recovered response matrices, in this section, we analyze the correct recovery of the mixing matrix $\mathbf{A}$ and the observed causal matrix $\mathbf{B}$, with  interventions. All the proofs are in the Appendix. 


Since the optimization in Algorithm \ref{alg:1} is non-convex, we analyze a variant of Algorithm \ref{alg:1}. This variant involves recovering 
each of $\mathbf{C}$ and $\mathbf{D}_i$ separately using tensor methods using Eq. \ref{eq:ica_tensor}, and finding a permutation on columns of $\mathbf{C}$ and $\mathbf{D}_i$ such that is minimizes the linear system \ref{eq:opt_1} - \ref{eq:opt_2}. 
\begin{align}
    \min_{\mathbf{C}} \| K_4(\mathbf{C}) - \mathbf{M}_4(\mathbf{X}) \|_F^2   , \hspace{1em}  \min_{\mathbf{D}_i}   \| K_4(\mathbf{D}_i) - \mathbf{M}_4(\mathbf{Y}_i) \|_F^2 \label{eq:ica_tensor}
\end{align} and using the recovered $\mathbf{C}$ and $\mathbf{D}_i$, we optimize the following
\begin{subequations}
\begin{align}
\label{eq:opt_1}
\min_{\mathbf{A},\mathbf{B}, \text{perm}(\mathbf{C}), \text{perm}(\mathbf{D}_i)} \ \  &  \left\|(\mathbf{I}-\mathbf{B})\mathbf{C}-\mathbf{A}\right\|_F^2 + \sum_{i=1}^{\nn} \left\|(\mathbf{I}-\mathbf{B}_i)\mathbf{D}_i-\mathbf{A}_i\right\|_F^2\\
\text{s.t. } \ \ & \mathrm{rank}\left(\mathbf{A}\right) = \min\{\nm,\nn\},
\qquad \mathbf{A}\in\mathbbm{R}^{\nn\times\nm},\label{eq:opt_2}
\qquad\mathbf{B}\in\mathbbm{R}^{\nn\times\nn},
\end{align}
\end{subequations} where $perm(\cdot)$ represents permutation of the columns. Thus, we find an optimal permutation that minimizes (\ref{eq:opt_2}). The above optimization is computationally more expensive than Algorithm 1 since it requires search over all permutations. But it allows us to provide theoretical guarantees on recovery.

\vspace{-1mm}
We   now provide theoretical guarantees for this variant. 
\subsection{Recovery with infinite samples }


We first analyze the case when infinite number of samples are available to the algorithm. In this case, the following theorem holds.
\begin{theorem}[Unique recovery with infinite samples]\label{thm:1} There exists a unique permutation of columns of  matrices $D_i$, recovered from tensor decomposition methods, such that $(\mathbf{I}-\mathbf{B}_i)\mathbf{D}_i=\mathbf{A}_i$ for all $i$.
The mixing matrix $\mathbf{A}\in\mathbbm{R}^{n\times m}$ and the observed causal matrix $\mathbf{B}\in\mathbbm{R}^{\nn\times \nn}$ are uniquely recovered under the assumptions that $\mathbf{A}$ has full column rank and $\mathbf{B}$ is strict lower triangular.\end{theorem}

\subsection{Recovery with finite samples}

We analyze the number of data samples, denoted by $N$, required in Algorithm~\ref{alg:1}. Suppose we have no prior information and the subroutine of tensor method is performed in an unsupervised manner\footnote{Similar argument holds for the semi-supervised ICA learning set-up in~\cite{pmlr-v28-anandkumar13}.}. To be consistent with the settings of Algorithm 1 in~\cite{JMLR:v15:anandkumar14b}, the $4$-th order moments $\mathbf{M
}_4$ with respect to the response matrices $\mathbf{C},\mathbf{D}_1,\ldots,\mathbf{D}_{\nn}$ are estimated empirically from $N$ observed measurements. We assume the following conditions for the theorem to hold:

\noindent (1). Matrix $\mathbf{A} $ has full column rank and $\mathbf{B}$ is strict lower triangular.
\noindent (2). The entries of $\mathbf{H}$ are independent sub-Gaussian variables with $\smash{\mathbbm{E}[H_j^2]=1}$ and the $4$-th order cumulant $\textbf{M}_4$  is non-zero.
\noindent (3). The number of samples satisfies $N\geq \widetilde{\Omega}\left(\nn\nm^6\right)$.
\noindent (4). The initialization is performed by 4-th order generalization of SVD-based technique, with the number of initializations as $\smash{L \geq \nm^{\Omega(\nm^2/\nn^2)}}$.

Consider the estimated matrices $\widehat{\mathbf{A}}$ and $\widehat{\mathbf{B}}$. Let $\widehat{\mathbf{G}}:=(\mathbf{I}-\widehat{\mathbf{B}})^{-1}$. The next theorem states that given the conditions in above, the $\infty$-norm error for the recovered rows of $\widehat{\mathbf{A}}$ and columns of $\widehat{\mathbf{G}}$ is bounded from above.
\begin{theorem}[Sample Complexity]
Suppose the above conditions (1)-(3) hold. Then the rows of the estimated mixing matrix $\widehat{\mathbf{A}}$ and the columns of the estimated observed causal matrix $\widehat{\mathbf{G}}$ given by Algorithm~\ref{alg:1} satisfy that for $i=1,\ldots,\nn$, with high probability, 
\begin{align*}
&\left|\left|\mathbf{a}_i-\widehat{\mathbf{a}}_i\right|\right|_{\infty}\leq O\left(\frac{\nm^2\max\{\nn,\nm\}^3}{\sqrt{\nm}\min\{N,\sqrt{\nn^3 N}\}}\right)
\qquad \left|\left|\mathbf{g}^{(i)}-\widehat{\mathbf{g}}^{(i)}\right|\right|_{\infty}\leq O\left(\frac{\nm^2\max\{\nn,\nm\}^3}{\sqrt{\nn}\min\{N,\sqrt{\nn^3 N}\}}\right),
\end{align*}
\label{thm_2}when sign ambiguity and column permutation are taken into account.
\end{theorem}

\vspace{-10pt}




\subsection{Necessary Conditions on Interventions}

Consider interventions $\mathtt{Do}(X_i = \lambda_i)$ on the $\nn$ variables $X_1,\ldots,X_{\nn}$ denoted by  $\{\mathcal{I}_i: i\in \{0\}\cup\{1,\ldots,\nn\}\}$ where $\mathcal{I}_0$ is the null intervention of passive observations. The following necessary conditions on the number of interventions hold.

\begin{theorem}[Necessary conditions; SEM]
\label{thm:necessary}
There exists a graph $G=\left(\mathcal{X}\cup\mathcal{H},E\right)$ such that any collection interventions  $\{\mathcal{I}_i: i\in \mathcal{S}\}$ with $\max\{|\mathcal{S}|,m\}\leq \nn-1$ cannot recover its edge set $E$.
\end{theorem}

\begin{remark}
Our results indicate that even with an SEM model, if $m\leq \nn-1$, then the worst-case recovery of causal effects still requires $\nn$ interventions. In contrast, for graphical causal models with perfect data and faithfulness assumption similar to our Assumption~\ref{ass:1}, the authors of~\cite{Eberhardt2006} (see Proposition 2) have shown that $\nn-1$ interventions, counting the null intervention of passive observations (with each randomizing a single variable) are necessary to determine the causal relations of $\nn$ variables. 
\end{remark} 

\vspace{-5pt}


\section{Simulation}

We validate Algorithm~\ref{alg:1} by showing the trade-offs of the number of samples used and the accuracy of recovery. We generate each $\mathbf{H}_k$ i.i.d. according to a Laplacian distribution with mean $1$ for $k=1,\ldots,m$ and each $\mathbf{X}_{i}$ i.i.d. according to a Gaussian distribution with mean $0$ and variance $1$ for $i=1,\ldots,\nn$. The noise vector $\mathbf{N}$ has i.i.d. Gaussian entries with mean $0$ and variance $10^{-3}$. Figure~\ref{fig:sample_complexitiy} depicts the sample complexity results for recovering the observed causal matrix. In particular, we generate data from a linear system with different numbers of observables and hidden variables. We restart the optimization 20 times to find the best optimization error. We plot the best recovery error on $\mathbf{B}$ using Algorithm \ref{alg:1} on different number of variables when there are same number of hidden variables and the observables (3 variables: Figure \ref{fig:sample_complexitiy} Left, 4 variables: Figure \ref{fig:sample_complexitiy} Middle, 5 variables: Figure \ref{fig:sample_complexitiy} Right, 10 variables: Figure \ref{fig:ablation} Left). We can see that the estimation error on $\mathbf{B}$ decrease fast with more samples. We also see that the recovery error increases with the number of the observables and hidden variables.
\begin{figure}[t]
 \setlength{\tabcolsep}{-5pt}
\begin{tabular}{ccc}
      \includegraphics[height=3.5cm]{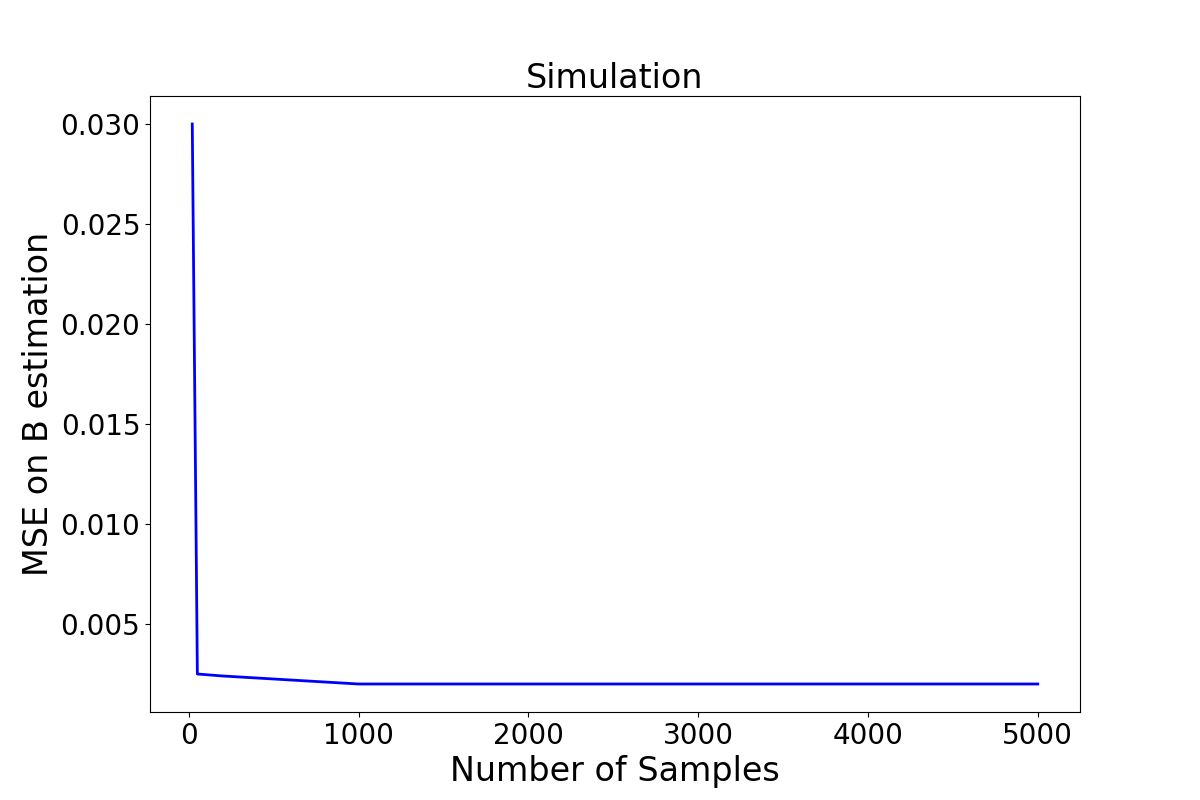}&
      \includegraphics[height=3.5cm]{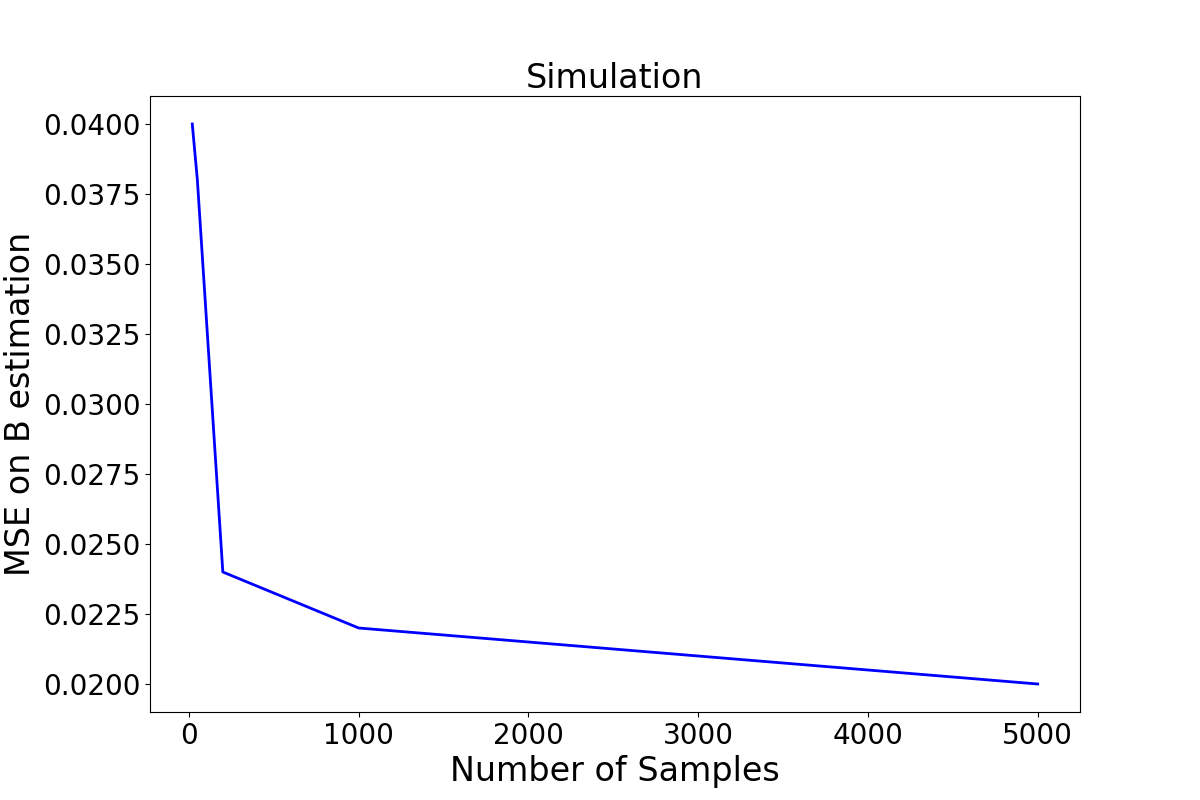}&
      \includegraphics[height=3.5cm, trim={1cm, 0cm, 0cm, 0cm},clip]{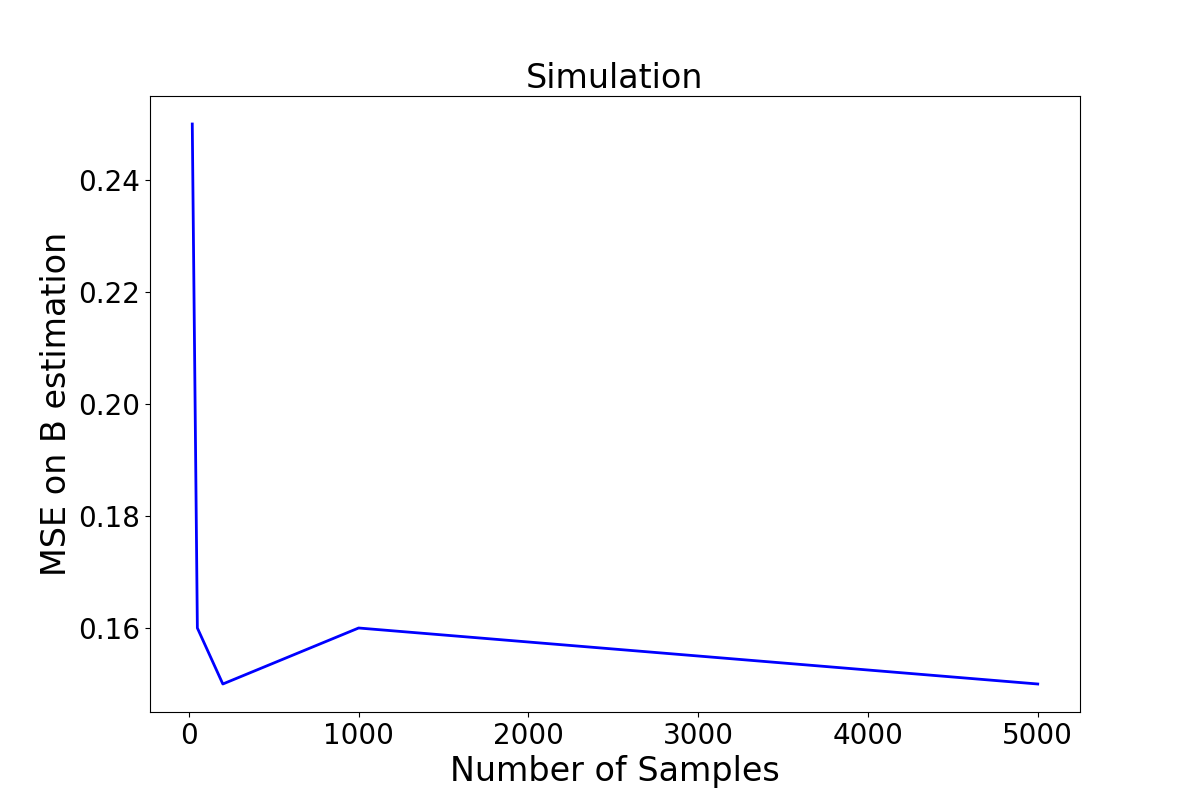}
    \end{tabular}
     \caption{Simulation results in 3 observable and 3 hidden ({\bf Left}), 4 observable and 4 hidden ({\bf Middle}) and 5 observable and 5 hidden ({\bf Right}). Trade-offs of the number of samples $N$ and Mean square error on $\mathbf{B}$. Note here the total number of samples are $(n+1)\times N$. Within the same mode, the error decrease fast with more samples. The recovery error increases with the number of the observables and hidden variables.}
    \label{fig:sample_complexitiy}
\end{figure}
\subsection{Ablation Study}
We conduct an ablation study on misspecified number of hidden variables and less number of interventions to see how robust our method is. Figure \ref{fig:ablation}(Middle) demonstrate that we can use less number of interventions to achieve similar empirical results. This indicates we only need small number of experiments instead of intervening on all the variables in practice. Figure \ref{fig:ablation} (Right) shows that when our model, in particular the number of latent variables, are slightly misspecified, our results are still quite robust. 
\begin{figure}[h]
 \setlength{\tabcolsep}{-5pt}
 \centering
\begin{tabular}{ccc}
      \includegraphics[height=3.5cm]{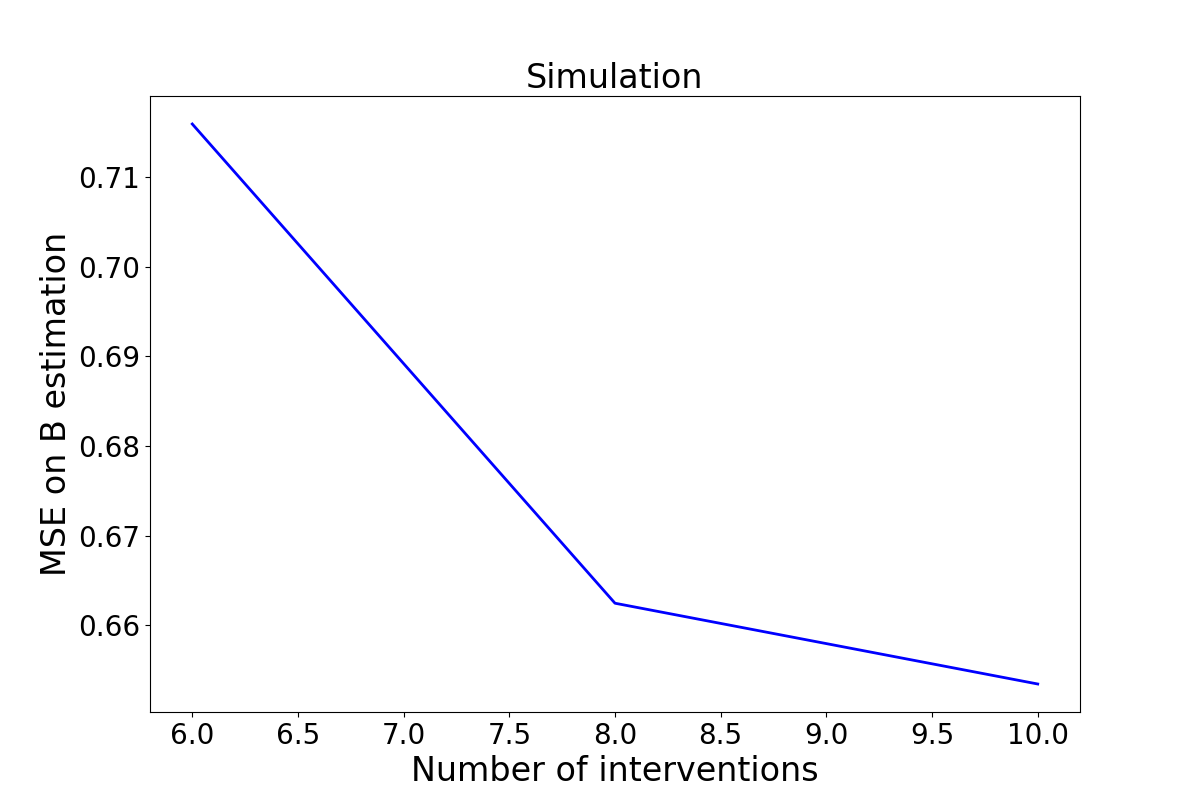}&
      \includegraphics[height=3.5cm]{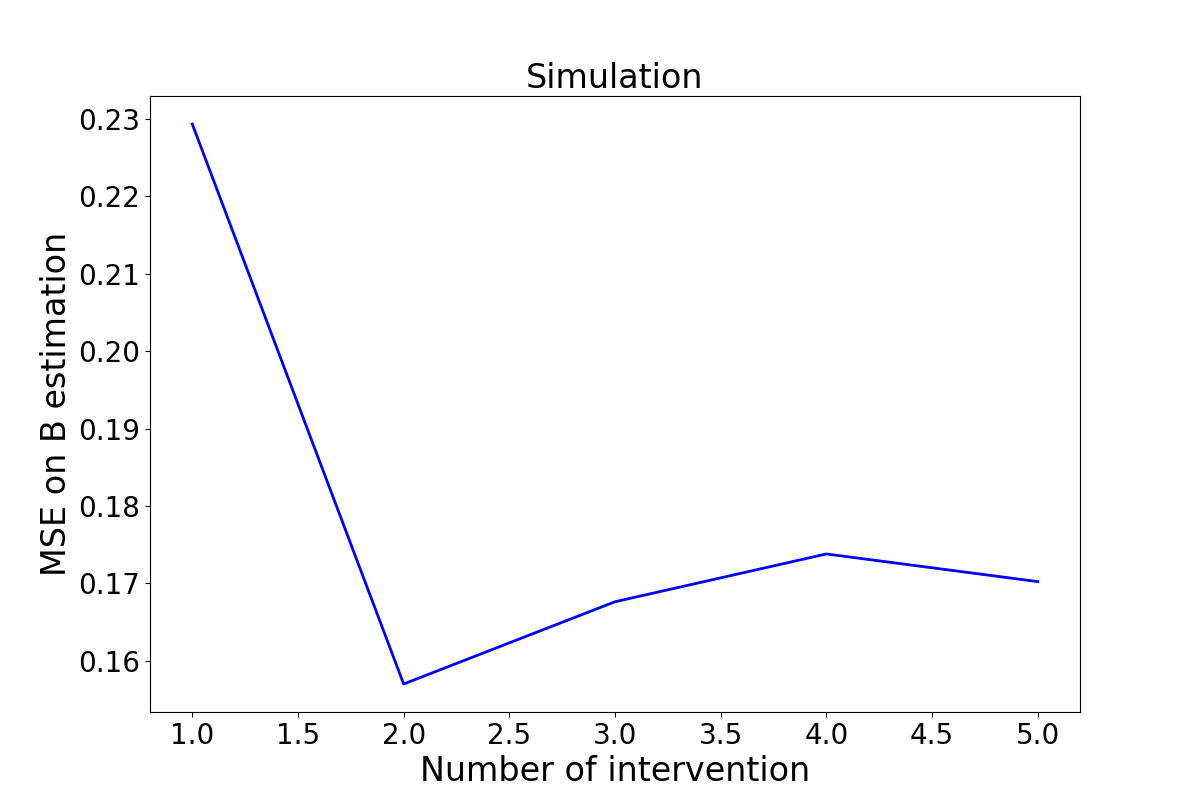}&
      \includegraphics[height=3.5cm]{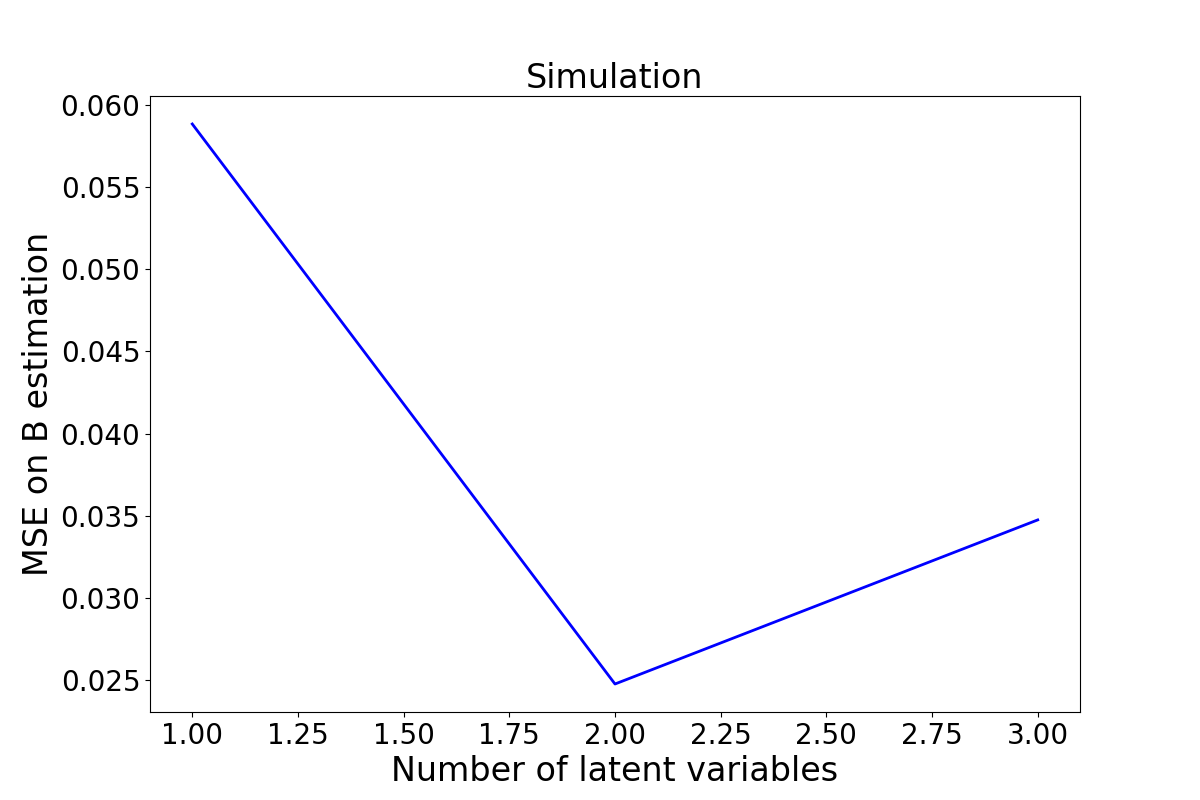}
    \end{tabular}
     \caption{ {\bf Left}: Simulation result on 10 observation variables with different numbers of interventions. {\bf Middle}: Performance with increasing number of interventions with a 4-variable model. We can observe that we can use less intervention to achieve similar empirical results with full interventions. {\bf Right}: Performance with misspecified number of latent variables on a 3-variable model. We can see that the model is relatively robust to small perturbations on number of latent variables.}
    \label{fig:ablation}
    \vspace{-0.3cm}
\end{figure}


\section{Related Work}
\label{sec:related}
{\bf Causal discovery without interventions: }
Traditionally, if the causal sufficiency assumption holds, \textit{i.e.,} no unobserved common causes exist for any pair of variables that are under consideration in the model, causal structure learning algorithms such as PC ~\cite{spirtes2000causation} and IC~\cite{Pearl2009} can identify a class of causal models that are equivalent to each other in terms of the conditional independence. Given more background knowledge, the models that fit the observed data can be narrowed down. 
In more general structure learning, assumptions about structures such as trees and polytress are made to learn from observational data \cite{sepehr2019algorithm,huangguaranteed}.
In this paper, we focus on modeling possible unobserved common causes, \textit{i.e.,} latent confounders, and establish a unique identification method using experimental data.  Utilizing hidden variables not only leads to a more descriptive representation of the model but also relaxes the computational complexity of many algorithms through dimensionality reduction. 

{\bf Causal discovery using interventions:}
Studying causal discovery using interventions has a long history. There is a large body of work on experimental design for casual discovery based on combinatorics and search~\cite{eberhardt2007causation,eberhardt2007interventions, hyttinen2014constraint}. Representing the confounding effects implicitly and allowing multi-variable interventions, a linear cyclic causal model can be identified with presence of latent confounders~\cite{hyttinen2012learning}. Recent work focuses on developing efficient randomized algorithm to learn a causal graph with confounding variables in a minimal or constrained budget \cite{kocaoglu2017experimental,acharya2018learning,ghassami2017budgeted,ghassami2019interventional,hauser2012characterization,eaton2007exact,rothenhausler2015backshift}. Combined with active learning and optimization theory, interventions can be regarded as active data queries with different cost \cite{he2008active,addanki2020efficient}.
In our case, we assume the invariance of latent confounders to relate mixing matrices across interventions. We restrict to the case of single-variable intervention and learn the parameters of the mixing and causal matrices explicitly.


{\bf Causal inference and discovery with latent variable models:} Previous research on causal learning and latent variable models focus on the recovering the causal order \cite{tashiro2014parcelingam}. Lingam model has also been generalized to hidden variable cases \cite{hoyer2008estimation} and latent Gaussian confounders \cite{chen2013causality}. Recent study of causal inference using latent variable models \cite{wang2019blessings} assumes the conditional independence between observable variables given the latent confounders. It opens up a line of research about utilizing latent variables in regression problems for estimating the potential outcome. In contrast to that, our problem setting is a causal discovery one, where we recover both the causal relationships between observable variables together with the confounding effects.

\section{Conclusion and Discussion}
\label{sec:conclusion}
In this work, we propose a novel approach to discover causal relation from experimental data based on recovering mixing and causal matrices in the latent linear structural equation model. We provide theoretical analysis regarding the conditions of unique identification and sample complexity for the method. The effectiveness of our approach is validated by experiments.
An interesting extension of the methods is to consider non-linear models including deep neural networks, where invariance of a mixing matrix $\mathbf{A}$ does not hold in general but need to be enforced by constraints.


\newpage

\begin{small}
\bibliography{bib}
\bibliographystyle{unsrt}
\end{small}

\appendix
\newpage

\section{Proof of Theorem~\ref{thm:1}}
\label{append_a}

\begin{proof}

Given $(\mathbf{I}-\mathbf{B})\mathbf{C}=\mathbf{A}$ and the linear equations obtained via $n$ interventions $\mathtt{Do}$($X_i = \lambda_i$) for $i=1,\ldots,\nn$:
\begin{align}
\label{app:4.1}
 (\mathbf{I}-\mathbf{B}_i)\mathbf{D}_i = \mathbf{A}_i,
\end{align}
we show that the matrices $\mathbf{A}$ and $\mathbf{B}$ are unique up to permutation and scaling.

Subtracting both sides of (\ref{app:4.1}) from $(\mathbf{I}-\mathbf{B})\mathbf{C}=\mathbf{A}$, for $i=1,\ldots,\nn$,
\begin{align*}
(\mathbf{I}-\mathbf{B})\mathbf{C} - (\mathbf{I}-\mathbf{B}_i)\mathbf{D}_i = \mathbf{A} - \mathbf{A}_i,
\end{align*}
which leads to
\begin{align*}
\mathbf{C}-\mathbf{D}_i =& \mathbf{A}-\mathbf{A}_i + \left(\mathbf{B}\mathbf{C}-\mathbf{B}_i \mathbf{D}_i\right)\\
=& \mathbf{A}-\mathbf{A}_i + \left(\mathbf{B}\mathbf{C}-\mathbf{B}\mathbf{D}_i+\mathbf{B}\mathbf{D}_i-\mathbf{B}_i\mathbf{D}_i\right).
\end{align*}
Rearranging the terms we obtain
\begin{align}
\label{eq:a.1}
\mathbf{C}-\mathbf{D}_i =& (\mathbf{I}-\mathbf{B})^{-1}\underbrace{\left[\left(\mathbf{A}-\mathbf{A}_i\right)+ (\mathbf{B}-\mathbf{B}_i)\mathbf{D}_i\right]}_{=:\mathbf{\underline{A}}_i}.
\end{align}
By the definitions of $\mathbf{A}_i$ and $\mathbf{B}_i$, the $\nn\times \nm$ matrix $\mathbf{A}-\mathbf{A}_i+ (\mathbf{B}-\mathbf{B}_i)\mathbf{D}_i$ is rank-$1$. The inverse matrix $(\mathbf{I}-\mathbf{B})^{-1}$ is full-rank and therefore, $\mathrm{rank}\left(\mathbf{C}-\mathbf{D}_i\right)=1$ for all $i=1,\ldots,\nn$. 
Denote by $\mathbf{G}:= (\mathbf{I}-\mathbf{B})^{-1}$ and let ${\mathbf{g}}^{(i)}$ be the $i$-th column of ${\mathbf{G}}$ and $\mathbf{a}_i$ be the $i$-th row of $\mathbf{\underline{A}}_i$. It follows that
$\mathbf{C}-\mathbf{D}_i = {\mathbf{g}}^{(i)} \mathbf{a}_i^{T}$. The vectors  ${\mathbf{g}}^{(i)}$ and $\mathbf{a}_i$ can be recovered using the singular value decomposition (SVD). To see the uniqueness, note that the recovered $\mathbf{A}'$ and $\mathbf{B}'$ need to satisfy~\eqref{app:4.1} for all $i=1,\ldots,\nn$ and $(\mathbf{I}-\mathbf{B})\mathbf{C}=\mathbf{A}$. To satisfy $(\mathbf{I}-\mathbf{B}')\mathbf{C}=\mathbf{A}'$, there must exist an orthogonal matrix $\mathbf{U}$ such that $\mathbf{A}'=\mathbf{A}\mathbf{U}$ and $(\mathbf{I}-\mathbf{B}')^{-1} = \mathbf{U}^{-1}(\mathbf{I}-\mathbf{B})^{-1}$, equivalently, $\mathbf{B}'=\mathbf{B}\mathbf{U}$ where $\mathbf{A}$ and $\mathbf{B}$ are the true matrices. However, it is not possible that $\mathbf{A}'_i=\mathbf{A}_i\mathbf{U}$ and $\mathbf{B}'_i=\mathbf{B}_i\mathbf{U}$ for every $i=1,\ldots,\nn$.
\end{proof}

\section{Proof of Theorem~\ref{thm_2}}
\label{append_b}

\begin{proof}
Consider the recovery of $\mathbf{D}_0,\ldots,\mathbf{D}_\nn$ using tensor decomposition described in Section~\ref{sec:tensor}. Assuming the conditions (1)-(5) presented above, Theorem $7$ in~\cite{anandkumar2014sample} guarantees that for all $i=0,\ldots,\nn$ with high probability, for all $j=1,\ldots,\nm$,
\begin{align}
\label{eq:5.3}
\mathrm{dist}\left(\mathbf{d}_{i,j},\widehat{\mathbf{d}}_{i,j}\right)\leq \varepsilon\left(\nm,\nn,N\right)
\end{align}
where for any two vectors $\mathbf{x},\mathbf{y}\in \mathbbm{R}^\nn$, the distance between them is defined as $$\underset{\mathbf{z}\perp \mathbf{y}}{\sup} {{\left\langle\mathbf{z},\mathbf{x}\right\rangle}/{(\|\mathbf{z}\|\cdot\|\mathbf{x}\|)}},$$
$\mathbf{d}_{i,j}$ represents the $j$-th column of the matrix $\mathbf{D}_i$ and the error term $\varepsilon\left(\nm,\nn,N\right)$ is defined as
\begin{align}
\label{eq:5.5}
\varepsilon\left(\nm,\nn,N\right):=\widetilde{O}\left(\frac{\nm^2}{\min\left\{N,\sqrt{\nn^3 N}\right\}}\right)+\widetilde{O}\left(\frac{\sqrt{\nm}}{\nn^{3/2}}\right).
\end{align}
Condition (6) guarantees that for $j=1,\ldots,\nm$,
\begin{align}
\label{eq:5.2}
\left|\left|\mathbf{d}_{i,j} - \widehat{\mathbf{d}}_{i,j} \right|\right|_{2}\leq \sqrt{2}\mathrm{dist}\left(\mathbf{d}_{i,j},\widehat{\mathbf{d}}_{i,j}\right).
\end{align}
Thus, denote by $\mathbf{F}_i:=\mathbf{D}-\mathbf{D}_i$ and $\widehat{\mathbf{F}}_i:=\widehat{\mathbf{D}}-\widehat{\mathbf{D}}_i$, the rank-$1$ matrices in the linear systems presented in (\ref{eq:a.1}). Tensor decomposition gives
\begin{align}
\nonumber
\left|\left|\mathbf{f}_{i,j} - \widehat{\mathbf{f}}_{i,j} \right|\right|_{2} = & \left|\left|{\mathbf{d}}_{0,j} - {\mathbf{d}}_{i,j} - \left(\widehat{\mathbf{d}}_{0,j} - \widehat{\mathbf{d}}_{i,j}  \right)\right|\right|_{2}\\
\label{eq:5.1}
\leq & \left|\left|{\mathbf{d}}_{0,j} - \widehat{\mathbf{d}}_{0,j} \right|\right|_{2} + \left|\left| \left({\mathbf{d}}_{i,j} - \widehat{\mathbf{d}}_{i,j}  \right)\right|\right|_{2}
\end{align}
where we have used the triangle inequality to derive (\ref{eq:5.1}). Plugging in (\ref{eq:5.2}) and (\ref{eq:5.3}) into  (\ref{eq:5.1}), it follows that for all $i=0,\ldots,\nn$ with high probability, for all $j=1,\ldots,\nm$,
\begin{align}
    \nonumber
    \left|\left|\mathbf{f}_{i,j} - \widehat{\mathbf{f}}_{i,j} \right|\right|_{2}\leq 2\sqrt{2}\mathrm{dist}\left(\mathbf{d}_{i,j},\widehat{\mathbf{d}}_{i,j}\right)\leq 2\sqrt{2}\varepsilon\left(\nm,\nn,N\right).
\end{align}
Therefore, for each matrix $\mathbf{F}_i$, with high probability, we can bound the $\ell_\infty$-norm by its Frobenius norm from above as
\begin{align*}
    &\left|\left|\mathbf{F}_{i} - \widehat{\mathbf{F}}_{i} \right|\right|_{\infty} \leq    \sqrt{\nm}\left|\left|\mathbf{F}_{i} - \widehat{\mathbf{F}}_{i} \right|\right|_{F}\leq \sqrt{\nn\nm}\varepsilon\left(\nm,\nn,N\right).
\end{align*}
Since $\mathbf{F}_i$ is rank-$1$, there is only one positive singular value, denoted by $\sigma_i>0$ and $\mathbf{F}_i$ can be written as $\mathbf{F}_i=\sigma_i\mathbf{u}_i\mathbf{v}_i^{T}$. Applying the perturbation bound for SVD in Theorem 1.1 of \cite{fan2018eigenvector}, for all $i=0,\ldots,\nn$ with high probability, 
\begin{align*}
   \left|\left|\mathbf{a}_{i} - \widehat{\mathbf{a}}_{i}\right|\right|_{\infty}\leq& O(\left(\max\{\nn||\mathbf{u}||_1,\nm||\mathbf{v}||_1\}\right)^2)\\
   &\cdot \frac{\max\{\nn,\nm\}}{\sigma_i\sqrt{\nm}}\varepsilon\left(\nm,\nn,N\right),\\
   \left|\left|\mathbf{g}_{i} - \widehat{\mathbf{g}}_{i}\right|\right|_{\infty}\leq& O(\left(\max\{\nn||\mathbf{u}||_1,\nm||\mathbf{v}||_1\}\right)^2)\\
   &\cdot \frac{\max\{\nn,\nm\}}{\sigma_i\sqrt{\nn}}\varepsilon\left(\nm,\nn,N\right),
\end{align*}
yielding the desired results with $\varepsilon\left(\nm,\nn,N\right)$ defined in (\ref{eq:5.5}).
\end{proof}

\section{Proof of Theorem~\ref{thm:necessary}}
\label{app:proof_of_thm_necessary}
\begin{proof}[Proof of Theorem~\ref{thm:necessary}]

Suppose $G=(\mathcal{X}\cup\mathcal{H},E)$ is an arbitrary graph with a mixing matrix $\mathbf{A}$ and a causal matrix $\mathbf{B}$. For any $i\in\mathcal{S}$ with $|\mathcal{S}|\leq \nn-1$, the intervened SEM model in~\eqref{eq:intervened_system} is
\begin{align}
\label{eq:SEM_1}
\mathbf{X}_i = & \mathbf{A}_i \mathbf{H} + \mathbf{B}_i\mathbf{X}_i.
\end{align}
We show that there exists another graph $G'\neq G$ with a mixing matrix $\mathbf{A}$ and a causal matrix $\mathbf{B}'$ such that $G'$ is consistent with the SEM in~\eqref{eq:SEM_1}. The two matrices $\mathbf{B},\mathbf{B}'\in\mathbbm{R}^{\nn\times\nn}$ satisfy that $\mathrm{supp}(\mathbf{B})\neq\mathrm{supp}\left(\mathbf{B}'\right)$ and  $\dim\left(\mathrm{kernel}(\mathbf{B}-\mathbf{B}')\right)=\nn-1$. Moreover, $\mathbf{X}_i=(\mathbf{I}-\mathbf{B}_i)^{-1}\mathbf{A}_i\mathbf{H}$ implies $\dim\left(\mathrm{span}(\{\mathbf{X}_i:i\in\mathcal{S}\})\right)\leq\max\{|\mathcal{S}|,m\}$. Therefore $\max\{|\mathcal{S}|,m\}\leq\nn-1$ ensures that
$\mathbf{X}_i\in\mathrm{kernel}(\mathbf{B}-\mathbf{B}')$ for all $i\in\mathcal{S}$ with $|\mathcal{S}|\leq \nn-1$. Hence,
\begin{align*}
\mathbf{X}_i = & \mathbf{A}_i \mathbf{H} + \mathbf{B}_i\mathbf{X}_i + \left(\mathbf{B}'_i-\mathbf{B}_i\right)\mathbf{X}_i\\
=&\mathbf{A}_i \mathbf{H} + \mathbf{B}'_i\mathbf{X}_i,
\end{align*}
yielding that at least $\nn$ interventions (including the null intervention) are necessary.

\end{proof}

\end{document}